\documentclass[preprint,12pt]{elsarticle}

\usepackage{amssymb}
\usepackage{amsthm}

\journal{Elsevier journal}

\usepackage{latexsym}
\usepackage[fleqn]{amsmath}
\usepackage{amsfonts}
\usepackage[mathscr]{eucal}
\def\R{\mathbb{R}}
\def\E{\mathbb{E}}
\def\P{\mathbb{P}}
\def\y{{\boldsymbol{y}}}
\def\x{{\boldsymbol{x}}}
\def\b{{\boldsymbol{b}}}

\def\0{{\bf{0}}}
\def\1{{\bf{1}}}
\def\X{{\bf{X}}}
\def\I{{\bf{I}}}
\def\H{{\bf{H}}}
\def\vmu{{\boldsymbol{\mu}}}
\def\vvarepsilon{{\boldsymbol{\varepsilon}}}
\def\evmu{\widehat{\boldsymbol{\mu}}}
\def\tvmu{\widetilde{\boldsymbol{\mu}}}
\def\emu{\widehat{\mu}}

\def\ebv{\widehat{\boldsymbol{b}}}
\def\evbeta{\widehat{\boldsymbol{\beta}}}
\def\tvbeta{\widetilde{\boldsymbol{\beta}}}
\def\vbeta{\boldsymbol{\beta}}
\def\ebeta{\widehat{\beta}}

\def\eb{\widehat{b}}

\def\ed{\widehat{d}}

\def\alphaopt{\alpha_{\rm opt}}

\def\esigma{\widehat{\sigma}}
\def\ek{\widehat{k}}
\def\ealpha{\widehat{\alpha}}
\def\tr{{\rm trace}}

\def\eB{\widehat{B}}
\def\eR{\widehat{R}}
\def\DF{{\rm DF}}
\def\eXv{\widehat{\bf X}}
\def\eS{\widehat{S}}
\def\eSv{\widehat{\bf S}}
\def\eHv{\widehat{\bf H}}
\def\ePv{\widehat{\bf P}}
\def\eqv{\widehat{\boldsymbol{q}}}

\newtheorem{theorem}{Theorem}
\newtheorem{lemma}{Lemma}
\newtheorem{property}{Property}

\begin{document}

\begin{frontmatter}



\title{On an improvement of LASSO by scaling}


\author{Katsuyuki Hagiwara}

\ead{hagi@edu.mie-u.ac.jp}
\address{Faculty of Education, Mie University,\\ 1577 Kurima-Machiya-cho, Tsu, 514-8507, Japan}

\begin{abstract}

A sparse modeling is a major topic in machine learning and
statistics. LASSO (Least Absolute Shrinkage and Selection Operator) is a
popular sparse modeling method while it has been known to yield
unexpected large bias especially at a sparse representation. There have
been several studies for improving this problem such as the introduction
of non-convex regularization terms. The important point is that this
bias problem directly affects model selection in applications since a
sparse representation cannot be selected by a prediction error based
model selection even if it is a good representation. In this article, we
considered to improve this problem by introducing a scaling that expands
LASSO estimator to compensate excessive shrinkage, thus a large bias in
LASSO estimator. We here gave an empirical value for the amount of
scaling. There are two advantages of this scaling method as follows. Since
the proposed scaling value is calculated by using LASSO estimator, we only
need LASSO estimator that is obtained by a fast and stable optimization
procedure such as LARS (Least Angle Regression) under LASSO modification
or coordinate descent.  And, the simplicity of our scaling method
enables us to derive SURE (Stein's Unbiased Risk Estimate) under the
modified LASSO estimator with scaling.  Our scaling method together with model
selection based on SURE is fully empirical and do not need additional
hyper-parameters. In a simple numerical example, we verified that our
scaling method actually improves LASSO and the SURE based model
selection criterion can stably choose an appropriate sparse model.
\end{abstract}


\end{frontmatter}



\section{Introduction}

A sparse modeling is a major topic in machine learning and
statistics. Especially, LASSO (Least Absolute Shrinkage and Selection
Operator) is a popular method that has been extensively
studied\cite{LARS,KF2000,FL2001,LASSO,HZ2006,NM2007,ZHT2007}. LASSO is
an $\ell_1$ penalized least squares method and has a nature of
soft-thresholding that implements thresholding and shrinkage of
coefficients; see \cite{DJ1994,DJ1995}. These two properties are
simultaneously controlled by a single regularization parameter. This
causes an excessive shrinkage, thus, a large bias that is directly
related to a consistency of model selection by LASSO. This has been
pointed out by \cite{LLW2006,FL2001,NM2007} and it has been proposed
several methods for solving this
problem\cite{FL2001,HZ2006,NM2007,MCP}. \cite{HZ2006} has proposed
adaptive LASSO that employs a weighted $\ell_1$ penalty, by which small
penalty is assigned to a large coefficient values.  \cite{NM2007} has
proposed relaxed LASSO to solve a limitation of one parameter control for
thresholding and shrinkage by introducing an additional parameter. On
the other hand, \cite{FL2001} has proposed SCAD (Smoothly Clipped
Absolute Deviation) that employs a non-convex penalty instead of
$\ell_1$ penalty. \cite{MCP} has also introduced a different type of
non-convex penalty called MCP(minimax concave penalty). The introduction
of non-convex penalty has an effect to suppress a bias at large values
of estimators. Since the methods with non-convex penalty have a
difficulty in optimization, the solutions to them have been
investigated; e.g. \cite{ZZ2012,LW2015}. \cite{ZZ2012} has shown that a
gradient descent started from a LASSO solution yields a local minimum of
a objective function with a non-convex penalty and it can be a good
solution for the objective function.  In \cite{LW2015}, local minima in
a non-convex penalty method including SCAD and MCP have good quality for
true values of coefficients.

In this article, we focus on a model selection problem in applications
of a sparse modeling. \cite{ZY2006} has shown that LASSO has a
consistent model selection property under a certain condition that is,
however, known to be somewhat restrictive. On the other hand, under
milder conditions than in LASSO, adaptive LASSO, SCAD and MCP have the
oracle property that consists of consistency of model selection and
asymptotic normality of the estimators of non-zero
coefficients\cite{FL2001,MCP}. All these results are based on an
appropriate setting of the regularization parameter. Therefore, it does
not tell us a choice of the regularization parameter in
application. Usually, it relies on the cross validation; e.g. it is
commonly implemented in many software packages. We emphasize that a bias problem
in LASSO directly affects a choice of model (regularization parameter)
under the cross validation. Since a bias is high at a sparse
representation in LASSO, a good sparse representation may not be
selected by a prediction error based criterion such as cross validation
error; see \cite{LLW2006}. We need to take into account of this point
rather than improvement of estimators. This model selection problem is
relaxed in cross validation for adaptive LASSO, SCAD and MCP since a
bias problem of LASSO is improved in these methods.  However, despite of
a good quality of SCAD and MCP estimators as in \cite{ZZ2012,LW2015},
local minima and optimization problem may yield a fluctuation of
estimators among the training sets in cross validation. The impact of
this fluctuation on validation error may not be well
evaluated. Especially, since these methods need an another
hyper-parameter for specifying the shape of penalty term, we need to
conduct cross validation for grid search on two hyper-parameters.

In this article, we consider to introduce a scaling of LASSO estimator;
i.e. scalar times of LASSO estimator.  We here give an appropriate
empirical scaling value which actually improves the excessive shrinkage,
thus a large bias in LASSO. The empirical scaling value has a simple
form with LASSO estimator; i.e. LASSO estimator is plugged in to the
scaling value. Therefore, in our method, we just need LASSO estimator
that can be obtained by a fast and stable method such as LARS (Least
Angle Regression)\cite{LARS} under LASSO modification or coordinate
descent\cite{WL2008,FHHT2007}. This is a benefit of our method in
comparing with the other methods including non-convex methods. Moreover,
a simplicity of our scaling method enables us to derive its analytic
model selection criterion that is $C_p$-type criterion based on SURE
(Stein's Unbiased Risk Estimate). For a naive LASSO, SURE has already
been derived in \cite{ZHT2007}. Actually, we apply this result to derive
SURE for the LASSO with scaling.  However, it is not available for
adaptive LASSO, relaxed LASSO and SCAD. Although it is derived for MCP
under a specific condition, its effectiveness in applications is not
clear; e.g. many software packages that implement MCP employed cross
validation. On the other had, our scaling method is closely related to
adaptive LASSO and relaxed LASSO. Adaptive LASSO controls biases
componentwisely by coefficientwise weights in $\ell_1$ regularizer. The
weights are calculated based on the initial estimator such as the least
squares estimators. Note that we may need ridge estimators as the
initial estimator for stable training in applications. The cost function
including the weighted $\ell_1$ regularizer can be simply optimized by a
modified LARS-LASSO\cite{HZ2006}. On the other hand, in relaxed LASSO,
shrinkage and thresholding parameters are introduced differently and
those are simultaneously optimized by an algorithm based on LARS-LASSO.
Relaxed LASSO can be viewed as controlling bias independently of
threshold. In this point of view, in our scaling method, threshold is
achieved by LASSO and amount of shrinkage is controlled by scaling
value.  Although there have been derived some important asymptotic
results for adaptive LASSO and relaxed LASSO, it may be difficult to
derive an analytic solution to model selection.  On the other hand, as
an improvement of adaptive LASSO, multi-step adaptive LASSO has been
proposed in \cite{BM2008}; see also \cite{XX2015}. Multi-step adaptive
LASSO employ adaptive LASSO at each cycle, in which LASSO estimators are
employed as initial estimators in weights. Multi-step adaptive LASSO is
similar to our scaling method since both methods employ LASSO estimator
in the parameters for improving a bias problem of LASSO.
Unfortunately, the method of model selection has not been discussed for
multi-step adaptive LASSO. In conclusion, we can say that
possibility of deriving SURE is an another benefit of our scaling
method.

In section 2, we give a regression framework including LASSO and a
definition of risk with its Stein's formula. In section 3, we introduce
a scaling of LASSO estimator. Especially, we give a reasonable empirical
scaling value and derive a model selection criterion under the given
scaling value. In section 4, we verify our results in section
3 through a simple numerical experiment. It includes comparisons to the
other modeling method such as MCP and adaptive LASSO. Section 5 is
devoted for conclusions and future works.

\section{LASSO with scaling}

\subsection{Regression problem and LASSO}

Let $\x=(x_1,\ldots,x_m)$ and $y$ be explanatory variables and a
response variable, for which we have $n$ samples :
$\{(x_{i,1},\ldots,x_{i,m},y_i):i=1,\ldots,n\}$.  We define
$\x_j=(x_{1,j},\ldots,x_{n,j})'\in\R^n$ for $j=1,\ldots,m$, where $'$
stands for the transpose operator.  We define $\X=(\x_1,\ldots,\x_m)$
and $\y=(y_1,\ldots,y_n)'$.  In this article, we assume that $m\le n$
holds and $\x_1,\ldots,\x_m$ are linearly independent.  Therefore,
$\X'\X$ is not singular here. Let $\varepsilon_1,\ldots,\varepsilon_n$
be i.i.d. samples from $N(0,\sigma^2)$; i.e. normal distribution with
mean $0$ and variance $\sigma^2$. Thus, by defining
$\vvarepsilon=(\varepsilon_1,\ldots,\varepsilon_n)'$, $\vvarepsilon\sim
N(\0_n,\sigma^2\I_n)$, where $\0_n$ is an $n$-dimensional zero vector
and $\I_n$ is an $n\times n$ identity matrix. We assume
$\y=\vmu+\vvarepsilon$.  We therefore have $\vmu=\E\y$, where $\E$ is
the expectation with respect to the joint probability distribution of
$\y$. We consider a regression problem by $\X\b$, where
$\b=(b_1,\ldots,\b_m)$ is a coefficient vector.  Let
$\ebv=(\eb_1,\ldots,\eb_m)$ be an estimator of $\b$.  LASSO is a method
for obtaining coefficient estimators that minimize $\ell_1$ regularized
cost function defined by
\begin{equation}
\label{eq:cost-lasso}
C_{\lambda}(\b)=\|\y-\X\b\|^2+\lambda\|\b\|_1,
\end{equation}
where $\|\cdot\|$ is the Euclidean norm and
$\|\b\|_1=\sum_{k=1}^n|b_j|$. $\lambda\ge 0$ is a regularization
parameter.  The second term of the right hand side of
(\ref{eq:cost-lasso}) is called $\ell_1$ regularizer.  Let
$\ebv_{\lambda}=(\eb_{1,\lambda},\ldots,\eb_{m,\lambda})$ be a LASSO
solution. Since the LASSO is known to be yield a sparse representation
under an appropriate choice of $\lambda$, some of elements in
$\ebv_{\lambda}$ are exactly zeros. 
We denote a LASSO output vector by $\evmu_{\lambda}=(\emu_{\lambda,1},\ldots,\emu_{\lambda,n})'$
that is given by $\evmu_{\lambda}=\X\ebv_{\lambda}$.
We define $\eB_{\lambda}=\{i:\eb_{i,\lambda}\neq 0\}$ and
$\ek_{\lambda}=|\eB_{\lambda}|$. $\eB_{\lambda}$ is called an active
set. There are regularization parameter values at which the active set
changes. We denote those by $\lambda_0>\cdots>\lambda_J=0$, in which
$\ebv_{\lambda}=\0_m$ for $\lambda>\lambda_0$ under a given $\y$.
$\lambda_j$ is called a transition point.  

Let $\X_{\eB_{\lambda}}$ be an $n\times\ek_{\lambda}$ matrix whose
column vectors are $\x_j$, $j\in\eB_{\lambda}$. We write
$\eXv_{\lambda}=\X_{\eB_{\lambda}}$ for simplicity.  Also we define
$\evbeta$ as a $\ek$-dimensional vector whose elements are
$\{\eb_k:k\in\eB_{\lambda}\}$.  We write
$\evbeta_{\lambda}=(\ebeta_{1,\lambda},\ldots,\ebeta_{\ek,\lambda})'$;
i.e. $\ebeta_k$ is a member of $\{\eb_k:k\in\eB_{\lambda}\}$ under an
appropriate enumeration.  Under this definition, we have
$\evmu_{\lambda}=\eXv_{\lambda}\evbeta_{\lambda}$ since
$\eb_{k,\lambda}=0$ for $k\notin\eB_{\lambda}$.  Let
$\eSv_{\lambda}=(\eS_{1,\lambda},\ldots,\eS_{\ek,\lambda})'$ be a sign
vector of $\evbeta_{\lambda}$; i.e.
\begin{equation}
\eS_{k,\lambda}=\begin{cases}
       1 & \ebeta_{k,\lambda}>0\\
       0 & \ebeta_{k,\lambda}=0\\
       -1 & \ebeta_{k,\lambda}<0\\
      \end{cases}. 
\end{equation}

\subsection{Some facts on LASSO estimate}

By Lemma 1 in \cite{ZHT2007}, the LASSO estimator satisfies that 
\begin{equation}
\label{eq:lemma1-ZHT2007}
\evbeta_{\lambda}=(\eXv_{\lambda}'\eXv_{\lambda})^{-1}
\left(\eXv_{\lambda}'\y-\lambda\eSv_{\lambda}\right)
\end{equation}
if $\lambda$ is not a transition point. Therefore, we have
\begin{equation}
\label{eq:evmu-lemma1-ZHT2007}
\evmu_{\lambda}=\eHv_{\lambda}\y-\lambda\eqv_{\lambda},
\end{equation}
where $\eHv_{\lambda}=\eXv_{\lambda}(\eXv_{\lambda}'\eXv_{\lambda})^{-1}\eXv_{\lambda}'$ and
 $\eqv_{\lambda}=\eXv_{\lambda}(\eXv_{\lambda}'\eXv_{\lambda})^{-1}\eSv_{\lambda}$.
It is easy to check that $\eHv_{\lambda}$ is an idempotent matrix.

We define
$\tvbeta_{\lambda}=(\eXv_{\lambda}'\eXv_{\lambda})^{-1}\eXv_{\lambda}'\y$.
This is the least squared estimator under $\eXv_{\lambda}$, thus in a
post estimation. Note that this is not a linear estimator of $\y$ since
$\eXv_{\lambda}$ is already chosen according to $\y$.  We also define
$\tvmu_{\lambda}=\eXv_{\lambda}\tvbeta_{\lambda}$. Obviously, this can
be written as $\tvmu_{\lambda}=\eHv_{\lambda}\y$. Therefore,
(\ref{eq:evmu-lemma1-ZHT2007}) can be written as
\begin{equation}
\label{eq:evmu-lemma1-ZHT2007-2}
\evmu_{\lambda}=\tvmu_{\lambda}-\lambda\eqv_{\lambda}
\end{equation}
for a non-transition $\lambda$.
We summarize some facts that are derived by 
(\ref{eq:evmu-lemma1-ZHT2007}) and are used in this article.

\begin{lemma}
\label{lemma:several-equations}
If $\lambda$ is not a transition point, the following equations hold.
\begin{align}
\label{eq:evmu-eqv}
\evmu_{\lambda}'\eqv_{\lambda}&=\|\evbeta_{\lambda}\|_1\\
\label{eq:evmu2-evmuy}
\|\evmu_{\lambda}\|^2&=\evmu_{\lambda}'\y-\lambda\eqv_{\lambda}'\y
+\lambda^2\|\eqv_{\lambda}\|^2\\
\label{eq:evmu2-evmuy=l1}
\|\evmu_{\lambda}\|^2&=\evmu_{\lambda}'\y-\lambda\|\evbeta_{\lambda}\|_1\\
\label{eq:H-evmu}
\H_{\lambda}\evmu_{\lambda}&=\evmu_{\lambda}.
\end{align}
\end{lemma}

\begin{proof}
In this proof, we drop $\lambda$ from symbols for simplifying the
description of terms.  By (\ref{eq:evmu-lemma1-ZHT2007}), we have
\begin{align}
\evmu'\eqv&=\eSv'(\eXv'\eXv)^{-1}\eXv'\eXv(\eXv'\eXv)^{-1}(\eXv'\y-\lambda\eSv)
=\eSv'\evbeta.
\end{align}
We then obtain (\ref{eq:evmu-eqv}) by the definition of $\eSv$.

We define $\ePv=\I_n-\eHv$. By the definition of $\eHv$ and $\ePv$, we have
\begin{equation}
\label{eq:lemma-qHy=qy}
\eqv'\eHv\y=\eqv'\y
\end{equation}
and, thus,
\begin{equation}
\eqv'\ePv\y=\eqv'\y-\eqv'\eHv\y=0.
\end{equation}
Since $\eHv$ is an idempotent matrix, we have $\eHv\ePv=O_n$, where
$O_n$ is an $n\times n$ zero matrix.  Thus, by
(\ref{eq:lemma1-ZHT2007}), (\ref{eq:evmu2-evmuy}) is obtained as
\begin{align}
 \evmu'\y-\|\evmu\|^2
 &=\evmu'(\y-\evmu)\notag\\
&=(\eHv\y-\lambda\eqv)'(\ePv\y+\lambda\eqv)\notag\\
&=\lambda\eqv'\y-\lambda^2\|\eqv\|^2.
\end{align}
Moreover, by (\ref{eq:lemma-qHy=qy}), 
(\ref{eq:evmu-lemma1-ZHT2007}) and (\ref{eq:evmu-eqv}), we have
\begin{align}
\lambda\eqv'\y-\lambda^2\|\eqv\|^2
&=\lambda\eqv'\eHv\y-\lambda^2\|\eqv\|^2\notag\\
&=\lambda\eqv'(\eHv\y-\lambda\eqv)\notag\\
&=\lambda\eqv'\evmu\notag\\
&=\lambda\|\evbeta\|_1.
\end{align}

Finally, by the definition of $\eHv$ and $\evmu$, we obtain
\begin{align}
\eHv\evmu=\eXv(\eXv'\eXv)^{-1}\eXv'\eXv(\eXv'\eXv)^{-1}(\eXv'\y-\lambda\eSv)=\evmu.
\end{align}
\end{proof}

\subsection{Definition of risk and its Stein's formula}

Let $\evmu=(\emu_1,\ldots,\emu_n)'\in\R^n$ be a regression estimate of
$\vmu=\E\left[\y\right]$. A prediction capability of $\evmu$ is
measured by a risk :
\begin{equation}
R_n=\frac{1}{n}\E\left[\|\evmu-\vmu\|^2\right],
\end{equation}
where $\E$ is the expectation with respect to the joint probability
distribution of $\y$. It is easily verified that
\begin{align}
\label{eq:general-risk}
R_n&=\frac{1}{n}\E\left[\|\evmu-\y\|^2\right]-\sigma^2+\DF_n,
\end{align}
where
\begin{equation}
\label{eq:general-df}
\DF_n=\frac{2}{n}\E\left[(\evmu-\E\left[\evmu\right])'(\y-\vmu)\right]
\end{equation}
that is a covariance between $\evmu$ and $\y$.
$\DF_n$ is often called the degree of freedom.

Let $\partial\evmu/\partial\y$ be an $n\times n$ matrix whose
 $(i,j)$ entry is $\partial\emu_i/\partial y_j$. 
We define
\begin{align}
\nabla\cdot\evmu=\tr\frac{\partial\evmu}{\partial\y}=
\sum_{i=1}^n\frac{\partial\emu_i}{\partial y_i}
\end{align} 
in which $\tr$ denotes the trace of a matrix. In \cite{CS1981}, it has
been shown that
\begin{align}
\label{eq:general-df-stein}
\DF_n&=\frac{2\sigma^2}{n}\E\left[\nabla\cdot\evmu\right]
\end{align} 
holds if $\emu_i=\emu_i(\y):\R^n\mapsto\R$, $i=1,\ldots,n$ are almost
differentiable in the term of \cite{CS1981} and the expectation in the
right hand side exists. $\nabla\cdot\evmu$ is called a divergence of
$\evmu$. By this result,
\begin{equation}
\eR_n(\sigma^2)=-\sigma^2
\frac{1}{n}\|\evmu-\y\|^2+\frac{2\sigma^2}{n}\nabla\cdot\evmu
\end{equation}
is an unbiased estimator of a risk $R_n$.  $\eR_n(\sigma^2)$ is
called SURE (Stein's Unbiased Risk Estimate).  We can then construct a
$C_p$-type model selection criterion by replacing $\sigma^2$ with an
appropriate estimate $\esigma^2$; e.g. \cite{ZHT2007}.

\section{LASSO with scaling}

\subsection{An optimal scaling}

We now consider to assign a positive single scaling parameter to LASSO
estimator. More precisely, the scaling parameter is denoted by
$\alpha>0$ and the modified LASSO estimator with scaling is given by
$\alpha\vbeta_{\lambda}$, where $\vbeta_{\lambda}$ is a vector of
non-zero elements of LASSO estimator.  The output vector with a single
scaling parameter is given by
$\evmu_{\lambda,\alpha}=\alpha\evmu_{\lambda}$.  Thus,
$\evmu_{\lambda,1}$ is a LASSO output vector.  We write
$\evmu_{\lambda,\alpha}=(\emu_{\lambda,\alpha,1},\ldots,\emu_{\lambda,\alpha,n})'$,
where $\emu_{\lambda,\alpha,k}=\alpha\emu_{\lambda,k}$.

A risk of LASSO with scaling is 
\begin{equation}
R_n(\lambda,\alpha)=\frac{1}{n}\E\left[\|\evmu_{\lambda,\alpha}-\vmu\|^2\right].
\end{equation}
Especially, $R_n(\lambda,1)$ is a risk of LASSO.
By the previous discussion, it is given by
\begin{align}
\label{eq:risk-ls-1}
R_n(\lambda,\alpha)
&=\frac{1}{n}\E\|\evmu_{\lambda,\alpha}-\y\|^2-\sigma^2
+\DF_n(\lambda,\alpha),
\end{align}
where
\begin{equation}
\DF_n(\lambda,\alpha)=\frac{2}{n}\E(\evmu_{\lambda,\alpha}-\E\evmu_{\lambda,\alpha})'(\y-\vmu). 
\end{equation}
In \cite{ZHT2007}, for LASSO estimate, 
\begin{equation}
\label{eq:risk-ls-2}
\DF_n(\lambda,1)=\frac{2\sigma^2}{n}\E\ek_{\lambda}
\end{equation}
has been shown via the above Stein's formula. By the definition of $\evmu_{\lambda,\alpha}$, we
thus have
\begin{equation}
\label{eq:risk-ls-3}
R_n(\lambda,\alpha)
=\frac{1}{n}\E\|\alpha\evmu_{\lambda}-\y\|^2-\sigma^2+\frac{2\alpha\sigma^2}{n}\E\ek_{\lambda}.
\end{equation}
Of course, this reduces to a risk of LASSO when $\alpha=1$. By
(\ref{eq:risk-ls-3}), SURE for LASSO is given by
\begin{equation}
\label{eq:ube-risk-LASSO}
\eR_n(\lambda,\sigma^2)
=-\sigma^2+\frac{1}{n}\|\evmu_{\lambda}-\y\|^2+\frac{2\sigma^2}{n}\ek_{\lambda}.
\end{equation}

On the other hand, by setting the derivative of (\ref{eq:risk-ls-3})
with respect to $\alpha$ to zero, the minimizing scaling value of
$R_n(\lambda,\alpha)$ is given by
\begin{equation}
\label{eq:alphaopt}
\alphaopt=\frac{\E\evmu_{\lambda}'\y-\sigma^2\E\ek_{\lambda}}{\E\|\evmu_{\lambda}\|^2}
\end{equation}
if $\E\|\evmu_{\lambda}\|^2\neq 0$. If $\lambda$ is not transition point
then we have
\begin{equation}
\label{eq:alphaopt-1}
\alphaopt
=1+\frac{\lambda\E\|\evbeta_{\lambda}\|_1}{\E\|\evmu_{\lambda}\|^2}
-\frac{\sigma^2\E\ek_{\lambda}}{\E\|\evmu_{\lambda}\|^2}
\end{equation}
by (\ref{eq:evmu2-evmuy=l1}).
Through a simple calculation using (\ref{eq:risk-ls-3}) and
(\ref{eq:alphaopt}), we have
\begin{equation}
\label{eq:R1-Ralphaopt}
 R_n(\lambda,1)-R_n(\lambda,\alphaopt)
 =\frac{1}{n}(\alphaopt-1)^2\E\|\evmu_{\lambda}\|^2.
\end{equation}
Therefore, the optimal scaling value improves naive LASSO at
any $\lambda$. In case of an orthogonal design in a nonparametric
regression problem such as wavelet\cite{DJ1994,DJ1995}, it is shown in
\cite{KH2016a} that the right hand side of (\ref{eq:R1-Ralphaopt}) is
$O(n^{-1}\log n)$.

\subsection{Data-dependent empirical scaling value}

One choice of a scaling value in applications is
$(\evmu_{\lambda}'\y-\sigma^2\ek_{\lambda})/\|\evmu_{\lambda}\|^2$ that
is an empirical estimate of $\alphaopt$.  In LASSO,
$\evmu_{\lambda}=\0_n$ happens to occur when $\lambda$ is
large. Therefore, this scaling value may not be stable. Moreover, the
scaling value can be smaller than one depending on the noise variance.
Also, it is difficult to handle this estimate since $\ek_{\lambda}$
is a dis-continuous function of $\y$.  As an another choice, we may have
$\evmu_{\lambda}'\y/\|\evmu_{\lambda}\|^2$ that minimizes the squared
distance between $\y$ and $\alpha\evmu_{\lambda}$; i.e. it approaches
LASSO estimator to the least squares one. However, again, this may not be
stable. Then, for a stable scaling value, we consider
\begin{equation}
\label{eq:def-ealpha-1}
\ealpha=\frac{\evmu_{\lambda}'\y+\delta}{\|\evmu_{\lambda}\|^2+\delta},
\end{equation}
where $\delta$ is a fixed positive constant.  Note that $\delta$ is not a
tuning parameter (hyper-parameter) and is a constant for stabilizing
$\ealpha$. Therefore, it is set to be a small value, say, $10^{-6}$ in
applications. By (\ref{eq:evmu2-evmuy=l1}), we can write
\begin{equation}
\label{eq:def-ealpha-2}
\ealpha=1+\frac{\lambda\|\evbeta_{\lambda}\|_1}{\|\evmu_{\lambda}\|^2+\delta}
\end{equation}
for non-transition $\lambda$.  Therefore, $\ealpha\ge 1$ holds; i.e. it
really behaves as an expansion parameter. Moreover, $\ealpha\simeq 1$
for a small $\lambda$. This is a nice property since the bias problem in
LASSO is serious when $\lambda$ is large and is not essential when it is
small. We have three facts relating to $\ealpha$.
The first one shows an effect of the introduction of $\ealpha$.

\begin{property}
\label{lemma:smaller-RSS}
For a non-transition $\lambda$, 
\begin{equation}
\|\y-\tvmu_{\lambda}\|^2\le
\|\y-\evmu_{\lambda,\ealpha}\|^2\le\|\y-\evmu_{\lambda,1}\|^2
\end{equation}
holds.
\end{property}

\begin{proof}
The first inequality is obvious because $\tvmu_{\lambda}$ is the least squares
 solution under $\eXv_{\lambda}$; i.e. it is a projection of $\y$ onto a
 linear subspace determined by column vectors of $\eXv_{\lambda}$.
For simplicity, we define $m_2=\|\evmu_{\lambda,1}\|^2$ and 
$p_1=\lambda\|\evbeta_{\lambda}\|_1$. We then obtain
 \begin{align}
\label{eq:lemma-smaller-RSS} 
&\|\y-\evmu_{\lambda,\ealpha}\|^2\notag\\
&=\|\y-\ealpha\evmu_{\lambda,1}\|^2\notag\\
&=\|\y-\evmu_{\lambda,1}+\evmu_{\lambda,1}-\ealpha\evmu_{\lambda,1}\|^2\notag\\
&=\|\y-\evmu_{\lambda,1}\|^2+(1-\ealpha)^2m_2
+2(1-\ealpha)\evmu_{\lambda,1}'(\y-\evmu_{\lambda,1})\notag\\
&=\|\y-\evmu_{\lambda,1}\|^2+(1-\ealpha)^2m_2+2(1-\ealpha)p_1\notag\\
&=\|\y-\evmu_{\lambda,1}\|^2+(1-\ealpha)^2m_2-2(1-\ealpha)^2(m_2+\delta)\notag\\
&=\|\y-\evmu_{\lambda,1}\|^2-(1-\ealpha)^2(m_2+2\delta),
 \end{align} 
where we used (\ref{eq:evmu2-evmuy=l1}) in the fourth line and 
(\ref{eq:def-ealpha-2}) in the fifth line.
\end{proof}
Therefore, the introduction of $\ealpha$ surely reduces the residual sum
compared to a LASSO estimate. This implies that $\ealpha$ moves the
LASSO estimator toward the least squares estimator at each $\lambda$.
We here consider
\begin{equation}
\ed(\lambda)=(1-\ealpha)^2(m_2+2\delta).
\end{equation}
As found in (\ref{eq:lemma-smaller-RSS}), 
$\ed(\lambda)$ is the difference between residuals of naive LASSO and 
LASSO with scaling.  Note that this is a function of
$\lambda$ if the training data is given and $\X$ is determined.

\begin{property}
\label{lemma:ed-lambda}
For simplicity, we consider a specific case where $\delta=0$. We assume
that $\|\evbeta_{\lambda}\|_1\neq 0$ holds and $\lambda$ is a
 non-transition point.
Let $\rho_{\min}$ and $\rho_{\max}$ be the
minimum and maximum eigenvalues of $\X'\X/n$ and assume $\rho_{\min}>0$.
Then we have
\begin{equation}
\label{eq:lemma-ed-lambda} 
\frac{\lambda^2}{n\rho_{\max}}\le 
\ed(\lambda)\le
\frac{\lambda^2m^2}{n\rho_{\min}}.
\end{equation}
\end{property}

\begin{proof}
Since
\begin{equation}
\ed(\lambda)
  =\lambda^2\|\evbeta_{\lambda}\|_1^2\frac{\|\evmu_{\lambda,1}\|^2+2\delta}
  {(\|\evmu_{\lambda,1}\|^2+\delta)^2}
  =\frac{\lambda^2\|\evbeta_{\lambda}\|_1^2}{\|\evmu_{\lambda,1}\|^2}
\end{equation}
holds in case of $\delta=0$, we have
\begin{equation}
\frac{\lambda^2\|\evbeta_{\lambda}\|_1^2}
 {n\rho_{\max}\|\evbeta_{\lambda}\|^2}\le 
 \ed(\lambda)
 \le\frac{\lambda^2\|\evbeta_{\lambda}\|_1^2}
{n\rho_{\min}\|\evbeta_{\lambda}\|^2}.
\end{equation}
By the equivalence of the norms, this reduces to 
(\ref{eq:lemma-ed-lambda}), where we used $\ek_{\lambda}\le m$.
\end{proof}

Therefore, the introduction of $\ealpha$ improves the degree of fitting
to the given data especially when $\lambda$ is large; i.e. a sparse
situation. We next argue on a probabilistic behavior of $\ealpha$.

\begin{property}
 \label{lemma:ub-1-ealpha}
 \begin{equation}
\label{eq:ub-1-ealpha}
\E\left[\ealpha-1\right]\le \max\left(1/\delta,m^2/\rho_{\min}\right)\frac{\lambda}{\sqrt{n}}
\end{equation}
holds.
\end{property} 

\begin{proof}
Since the probability that a fixed $\lambda$ is a transition point is
zero as in \cite{ZHT2007}, $\lambda$ is assumed to not be a transition
pont below.  We define an event
$E=\left\{\|\evbeta_{\lambda}\|_1\le\theta_n\right\}$, where
 $\theta_n>0$. $E^C$ denotes the complement of $E$.
By (\ref{eq:def-ealpha-2}), we have
 \begin{align}
  \E\left[\ealpha-1|E\right]\le\frac{1}{\delta}\E\left[\left.\lambda\|\evbeta_{\lambda}\|_1\right|E\right]\le\lambda\theta_n/\delta.
 \end{align}
 We also have
 \begin{align}
  \E\left[\ealpha-1|E^C\right]
  &\le\E\left[\left.\frac{\lambda\|\evbeta_{\lambda}\|_1}{n\rho_{\min}\|\evbeta_{\lambda}\|^2}\right|E^C\right]\notag\\
  &\le\E\left[\left.\frac{\lambda\ek_{\lambda}^2}{n\rho_{\min}\|\evbeta_{\lambda}\|_1}\right|E^C\right]\notag\\
  &\le\frac{\lambda m^2}{n\rho_{\min}\theta_n}.
 \end{align}
Since $\E[\ealpha-1]=\E[\ealpha-1|E]\P[E]+\E[\ealpha-1|E^C]\P[E^C]$, 
 we have (\ref{eq:ub-1-ealpha}) by taking $\theta_n=1/(2\sqrt{n})$.
\end{proof}

We consider the case where $\rho_{\min}$ and $m$ are constants.  This is
a natural setting of a classical linear regression problem.  In this
case, by the above result, the expectation of the degree of expansion is
bounded above by $O(1/\sqrt{n})$. Therefore, the effect of expansion by
$\ealpha$ is small when $n$ is large and $\X$ is fixed. This is also
found in the previous result.

\subsection{Model selection criterion under empirical scaling}

Now, we consider to derive a $C_p$-type model selection criterion for
$\evmu_{\lambda,\ealpha}$. For this purpose, we derive an unbiased
estimate of a risk for $\evmu_{\lambda,\ealpha}$. To do this, by
(\ref{eq:risk-ls-1}), we need to calculate the degree of freedom of
$\evmu_{\lambda,\ealpha}$. We define it by
\begin{equation}
\DF_n^{\rm sca}(\lambda)=
\frac{2}{n}\E\left[(\evmu_{\lambda,\ealpha}-\E\left[\evmu_{\lambda,\ealpha}\right])'
(\y-\vmu)\right]. 
\end{equation}

\begin{theorem}
\label{theorem:DFsca}
We have
\begin{equation}
\DF_n^{\rm sca}(\lambda)=\frac{2\sigma^2}{n}\E\left[\ed_1+\ed_2\right],
\end{equation}
where
\begin{align}
\label{eq:def-ed_1}
 \ed_1&=(1-\ealpha)\frac{\|\evmu_{\lambda}\|^2-\delta}{\|\evmu_{\lambda}\|^2+\delta}\\
\label{eq:def-ed_2} 
\ed_2&=\ealpha\ek_{\lambda}.
 \end{align}
\end{theorem}

\begin{proof}
We drop $\lambda$ from expressions for simplicity since we fix $\lambda$
below. We thus write $\evbeta=\evbeta_{\lambda}$, $\eSv=\eSv_{\lambda}$, 
$\evmu_{\alpha}=\evmu_{\lambda,\alpha}$, $\eB=\eB_{\lambda}$ and $\ek=\ek_{\lambda}$.

We can write $\evmu_{\alpha}=\alpha\X_{\eB}\evbeta$.  Especially,
$\evmu_1$ is a LASSO output. For simplicity, we write $\evmu=\evmu_1$
below. We denote the $k$th member of $\evmu_{\alpha}$ by
$\emu_{\alpha,k}$.  In \cite{ZHT2007}, it is shown that, for any fixed
$\lambda$, $\emu_{1,k}:\R^n\mapsto\R$, $k=1,\ldots,n$ are almost
differentiable. By (\ref{eq:def-ealpha-1}),
$\ealpha\emu_{1,k}:\R^n\mapsto\R$ is calculated by arithmetic operations
of the components of $\y$ and $\evmu$. Therefore, $\ealpha\emu_{1,k}$
is almost differentiable since it essentially requires a coordinate-wise
absolutely continuity. As a result, Stein's lemma can be applied to
$\evmu_{\ealpha,k}$ and, by (\ref{eq:general-df-stein}), we have
\begin{align}
\DF_n^{\rm sca}(\lambda)&=\frac{2\sigma^2}{n}\E\left[\nabla\cdot\evmu_{\ealpha}\right],
\end{align} 
where
\begin{align}
\nabla\cdot\evmu_{\ealpha}=
\tr\frac{\partial\evmu_{\ealpha}}{\partial\y}=
\sum_{i=1}^n\frac{\partial\mu_{\ealpha,i}}{\partial y_i}.
\end{align} 
Since 
\begin{align}
\sum_{i=1}^n\frac{\partial }{\partial y_i}\emu_{\lambda,\ealpha,i}
=\sum_{i=1}^n\emu_{1,i}\left(\frac{\partial }{\partial y_i}\ealpha\right)
+\ealpha\sum_{i=1}^n\left(\frac{\partial }{\partial y_i}\emu_{1,i}\right),
\end{align}
holds, we have
\begin{equation}
\label{eq:nabla-evmu-ealpha}
\nabla\cdot\evmu_{\ealpha}=\evmu'\left(\frac{\partial\ealpha}{\partial\y}\right)+
\ealpha\nabla\cdot\evmu,
\end{equation}
where $\partial \ealpha/\partial \y$ is an $n$-dimensional vector whose
$i$th entry is $\partial \ealpha/\partial y_i$.
Since the probability that a fixed $\lambda$ is a transition point is
zero as in \cite{ZHT2007}, $\lambda$ is assumed to not be a
transition pont below.
 
For the second term of 
(\ref{eq:nabla-evmu-ealpha}), it has been shown in \cite{ZHT2007} that
\begin{equation}
\label{eq:d-evmu-y-ZHT2007}
\frac{\partial\evmu}{\partial\y}=\eHv
\end{equation}
by (\ref{eq:lemma1-ZHT2007}) and 
the local constancy of $\eqv$ under a fixed $\lambda$. And, we thus have
\begin{equation}
\label{eq:divergence-LASSO}
\nabla\cdot\evmu=\tr\eHv=\ek
\end{equation}
by the idempotence of $\eHv$. Therefore, the second term of
(\ref{eq:nabla-evmu-ealpha}) is equal to $\ed_2$.

We evaluate the first term in below. Since
\begin{equation}
\frac{\partial}{\partial y_k}\|\evmu\|^2
=\frac{\partial}{\partial y_k}\sum_{j=1}^n\emu_j^2
=2\sum_{j=1}^n\emu_j\frac{\partial\emu_j}{\partial y_k},
\end{equation}
we have
\begin{equation}
\label{eq:d-evmu2}
\frac{\partial\|\evmu\|^2}{\partial\y}
=2\left(\frac{\partial\evmu}{\partial\y}\right)\evmu=2\eHv\evmu=2\evmu
\end{equation}
by (\ref{eq:d-evmu-y-ZHT2007}) and (\ref{eq:H-evmu}).
On the other hand, we have
\begin{equation}
\label{eq:d-evmu-y}
\frac{\partial\evmu'\y}{\partial\y}
=\frac{\partial}{\partial\y}
\left\{\|\evmu\|^2+\lambda\eqv'\y-\lambda^2\|\eqv\|^2\right\}
=2\evmu+\lambda\eqv
\end{equation}
by (\ref{eq:d-evmu2}), (\ref{eq:evmu2-evmuy}) in Lemma
\ref{lemma:several-equations} and
local constancy of $\eqv$ as in \cite{ZHT2007}.

By (\ref{eq:d-evmu2}), (\ref{eq:d-evmu-y}) and (\ref{eq:evmu2-evmuy}) in
Lemma \ref{lemma:several-equations}, we have
\begin{align}
\frac{\partial\ealpha}{\partial\y}
&=\frac{
\left(\|\evmu\|^2+\delta\right)\frac{\partial}{\partial\y}\evmu'\y
 -(\evmu'\y+\delta)\frac{\partial}{\partial\y}\|\evmu\|^2}
 {\left(\|\evmu\|^2+\delta\right)^2}\notag\\
&=\frac{\left(\|\evmu\|^2+\delta\right)(2\evmu+\lambda\eqv)
 -2(\evmu'\y+\delta)\evmu}{\left(\|\evmu\|^2+\delta\right)^2}\notag\\
&=\frac{\left(\|\evmu\|^2+\delta\right)(2\evmu+\lambda\eqv)
 -2\ealpha\left(\|\evmu\|^2+\delta\right)\evmu}
 {\left(\|\evmu\|^2+\delta\right)^2}\notag\\
&=\frac{1}{\|\evmu\|^2+\delta}
\left\{2\evmu+\lambda\eqv-2\ealpha\evmu\right\},
\end{align}
where the third line comes from (\ref{eq:def-ealpha-1}).
Therefore, we obtain
\begin{align}
\evmu'\left(\frac{\partial\ealpha}{\partial\y}\right)
&=\frac{1}{\|\evmu\|^2+\delta}
\left\{2\|\evmu\|^2+\lambda\evmu'\eqv-2\ealpha\|\evmu\|^2\right\}\notag\\
&=\frac{1}{\|\evmu\|^2+\delta}
\left\{2\|\evmu\|^2+\lambda\|\evbeta\|_1-2\ealpha\|\evmu\|^2\right\}\notag\\
&=\frac{2}{\|\evmu\|^2+\delta}(1-\ealpha)\|\evmu\|^2+
\frac{\lambda\|\evbeta\|_1}{\|\evmu\|^2+\delta}\notag\\
&=\frac{2}{\|\evmu\|^2+\delta}(1-\ealpha)\|\evmu\|^2+
\frac{\evmu'\y+\delta-\delta-\|\evmu\|^2}{\|\evmu\|^2+\delta}\notag\\
&=\frac{2}{\|\evmu\|^2+\delta}(1-\ealpha)\|\evmu\|^2-(1-\ealpha)\notag\\
&=(1-\ealpha)\frac{\|\evmu\|^2-\delta}{\|\evmu\|^2+\delta}
\end{align}
where we used (\ref{eq:def-ealpha-1}) and 
 (\ref{eq:evmu-eqv}), (\ref{eq:evmu2-evmuy=l1}).
\end{proof}

We have two remarks on this theorem.
\begin{itemize}
 \item Our discussion is always applicable when $\X'\X$ is not
       singular.
 \item $\E[\ed_1]\le O\left(1/\sqrt{n}\right)$ by Lemma
\ref{lemma:ub-1-ealpha} since $|\ed_1|\le\ealpha-1$.
\end{itemize}
By this theorem, the risk for $\evmu_{\lambda,\ealpha}$ is given by
\begin{align}
 R_n^{\rm sca}(\lambda)
 &=\frac{1}{n}\E\left[\|\vmu-\evmu_{\lambda,\ealpha}\|^2\right]\notag\\
&=-\sigma^2 +\frac{1}{n}\E\left[\|\y-\evmu_{\lambda,\ealpha}\|^2\right]
+\DF_n^{\rm sca}(\lambda).
\end{align}
Therefore, SURE for LASSO with scaling is given by
\begin{equation}
\label{eq:ube-risk-LASSO-S}
\eR_n^{\rm sca}(\lambda,\sigma^2)= 
-\sigma^2
+\frac{1}{n}\|\y-\evmu_{\lambda,\ealpha}\|^2+\frac{2\sigma^2}{n}\left(\ed_1+\ed_2\right),
\end{equation}
where $\ed_1$ and $\ed_2$ are defined by (\ref{eq:def-ed_1}) and
(\ref{eq:def-ed_2}) respectively.

\subsection{Estimate of noise variance}

To compute a $C_p$-type model selection criterion based on SURE, we need
an appropriate estimate of $\sigma^2$. For estimating the noise variance in a
regression problem, \cite{CE1992} has recommended to apply
\begin{equation}
\label{eq:esigma2}
\esigma^2_{\rm CE}=\frac{\y'(\I_n-\H_{\gamma})^2\y} {\tr[(\I_n-\H_{\gamma})^2]},
\end{equation} 
where $\H_{\gamma}=\X(\X'\X+\gamma\I_n)^{-1}\X'$ with
$\gamma>0$. $\H_{\gamma}$ can be viewed as the hat matrix in a ridge
regression with a ridge parameter $\gamma>0$ or, equivalently,
an $\ell_2$ regularization with a regularization parameter $\gamma$. In
general, the $\ell_2$ regularization is introduced for better
generalization and stabilization. We need to carefully select the
parameter value for the former reason. However, since the purpose to
introduce $\gamma$ here is to stabilize an estimate of the noise
variance. Therefore, we just set $\gamma$ to a small value, say,
$10^{-6}$ in applications.  Especially, this is effective when $m$ is
large; i.e. when a colinearity problem arises under a full model.

\section{Numerical examples}

In this section, through a simple numerical example, we verify our result
on SURE for LASSO with scaling and compare our method with naive LASSO,
MCP and Adaptive LASSO. We refer to Adaptive LASSO as A-LASSO and LASSO
with scaling $\ealpha$ as LASSO-S. 

\subsection{Setting of experiments}

For $u\in\R$, we define
$g_{\tau}(u,\xi)=\exp\left\{(u-\xi)^2/(2\tau)\right\}$, where $\xi\in\R$
and $\tau>0$. Let $u_i$, $i=1,\ldots,n$ be equidistant points in
$[-5,5]$. Let $\{\xi_1,\ldots,\xi_m\}$ be a subset of
$\{u_1,\ldots,u_n\}$, where $m\le n$. We take $\xi_j=u_{(n/m)j}$,
$j=1,\ldots,m$ by assuming $n/m$ is an integer. We define $n\times m$
matrix $\X_1$ whose $(i,j)$ entry is $g_{\tau}(u_i,\xi_j)$; i.e. the
$j$th column vector of $\X_1$ is an output vector of
$g_{\tau}(\cdot,\xi_j)$.  Let $\X_2$ be a normalized version of $\X_1$;
i.e. the mean and squared norm of each column vector of $\X_2$ are
equal to zero and $n$ respectively.  By taking account of the intercept, we
construct a design matrix by $\X=(\1_n,\X_2)$. Therefore, we consider a
curve fitting problem using a linear combination of $m$ Gaussian basis
functions whose centers are input data points that are appropriately
chosen.  We generate $y_i$ by
$y_i=\sum_{k=1}^m\beta_k^*g_{\tau}(u_i,\xi_k)+\varepsilon_i$, where
$\varepsilon_i\sim N(0,\sigma^2)$.  We define $K^*=\{k|\beta_k^*\neq
0\}$ and consider the case where $|K^*|\ll m$. This corresponds to the
case that there exists an exact sparse representation; i.e. there is a
small true representation. 

\subsection{Verification of risk estimate}

In the first numerical experiment, we verify our theoretical result of
SURE for LASSO-S. We here refer to $\eR_n(\lambda,\esigma^2_{\rm CE})$
in (\ref{eq:ube-risk-LASSO}) and $\eR_n^{\rm sca}(\lambda,\esigma^2_{\rm
CE})$ in (\ref{eq:ube-risk-LASSO-S}) as SUREs of LASSO and LASSO-S
respectively; i.e.  the noise variance is replaced with $\esigma^2_{\rm
CE}$ defined in (\ref{eq:esigma2}). These are fully empirical and,
thus, can be applied as model selection criteria.

We set $n=100$, $m=50$, $\sigma^2=1$, $\tau=0.1$, $K^*=\{5,18,31,45\}$
and $(\beta_{5}^*,\beta_{18}^*,\beta_{31}^*,\beta_{45}^*)=(1,-2,2,-1)$;
i.e. $\xi_j$'s of non-zero coefficients are almost equally positioned.
We also set $\delta=1/n$ for LASSO-S and $\gamma=10^{-6}$ in calculating
$\esigma^2_{\rm CE}$.  We here consider two cases of $\tau=0.1$ and
$\tau=0.4$. In both cases, some Gaussian functions that are close to
each other are relatively correlated. However, $4$ Gaussian functions
with non-zero coefficients (components of a target function) are nearly
orthogonal in the former case while those are still correlated in the
latter case. This condition of correlation among components in a target
function affects the consistency of model selection of LASSO,
A-LASSO and MCP.

We here employ LARS-LASSO for calculating LASSO path\cite{LARS} and use
``lars'' package\cite{LARS} in R. Since the regularization parameter
corresponds to the number of un-removed coefficients, we here observe
the relationship between the number of un-removed coefficients and
risk. Since we know the true representation, we can calculate the actual
risk by the mean squared error between the true output and estimated
output. We repeat this procedure for $1000$ times and calculate averages
of actual risks and SUREs.

The averages of actual risks and SUREs of LASSO and LASSO-S are depicted
in Fig.\ref{fig:papsim03g}. The horizontal axis is an average of the
number of non-zero coefficients (members in active set) at the each step
in LARS-LASSO. Note that, at a fixed step of LARS-LASSO, the number of
non-zero coefficients may be different for $1000$ trials. Therefore, we
take an average of those; i.e. the horizontal axis corresponds to the
number of LARS-LASSO steps while we show the number of averages of
non-zero coefficients at the steps in the horizontal axis.  In
Fig.\ref{fig:papsim03g}, we depict the results at some specific steps
(not the results at all steps) for the clarity of graphs. We have some
remarks on these results.

\begin{itemize}
\item SURE is well consistent with the actual risk for both of LASSO and
      LASSO-S. Especially, the consistency for LASSO-S verifies Theorem
      \ref{theorem:DFsca}. 
      
\item When the number of non-zero coefficients is small ($\lambda$ is
      large), LASSO-S shows  a lower risk compared to LASSO.
      This is notably for $\tau=0.1$; i.e. components of a target
      function are nearly orthogonal.

 \item The number of non-zero coefficients at which an averaged risk is
      minimized is smaller for LASSO-S than LASSO.
      This is also notable for $\tau=0.1$.
\end{itemize}

As a result, we can expect that $\eR_n^{\rm sca}(\lambda,\esigma^2_{\rm
CE})$ can be a good selector of $\lambda$ in applications; i.e. it can
choose a sufficiently sparse model with low risk.

\begin{figure}[p]
 \begin{center}
   \begin{minipage}[t]{80mm}
     \begin{center}
      \includegraphics[width=80mm]{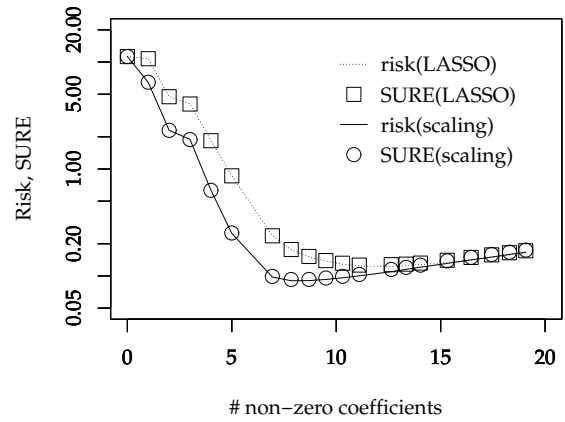}

      (a)~$\tau=0.1$
     \end{center}
   \end{minipage}

   \begin{minipage}[t]{80mm}
     \begin{center}
      \includegraphics[width=80mm]{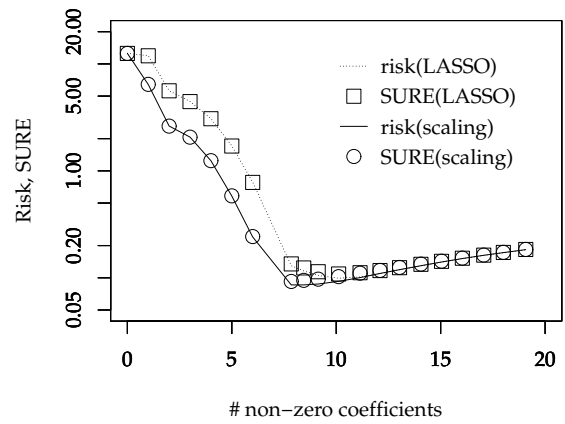}

      (b)~$\tau=0.4$
     \end{center}
   \end{minipage}

\caption{Averages of actual risks and SUREs for LASSO and LASSO-S.}
\label{fig:papsim03g}
 \end{center}
\end{figure}

\begin{figure}[p]
 \begin{center}
   \begin{minipage}[t]{80mm}
    \begin{center}
     \includegraphics[width=80mm]{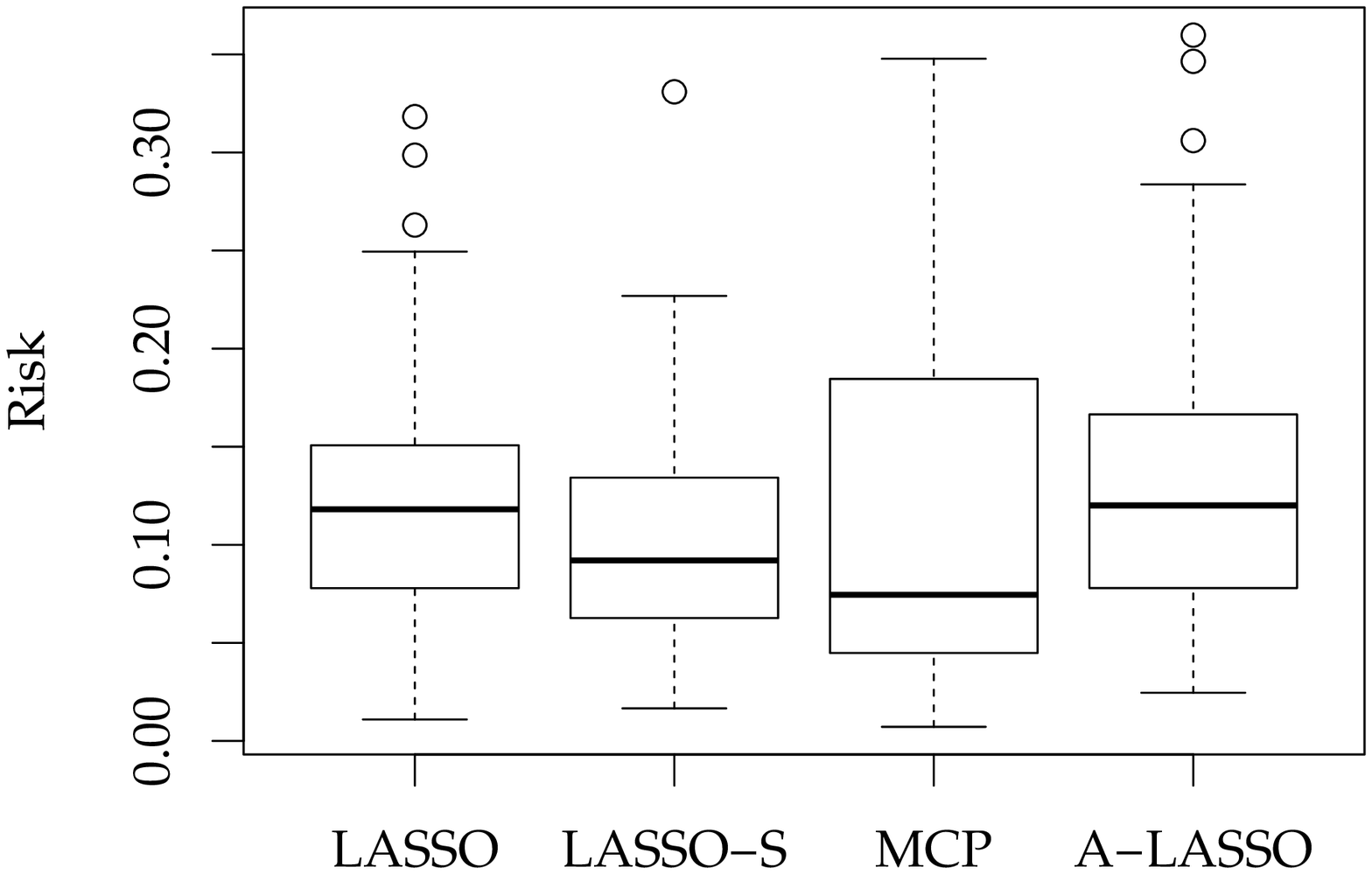}

     (a) $n=100$
    \end{center}
   \end{minipage}

   \begin{minipage}[t]{80mm}
    \begin{center}
     \includegraphics[width=80mm]{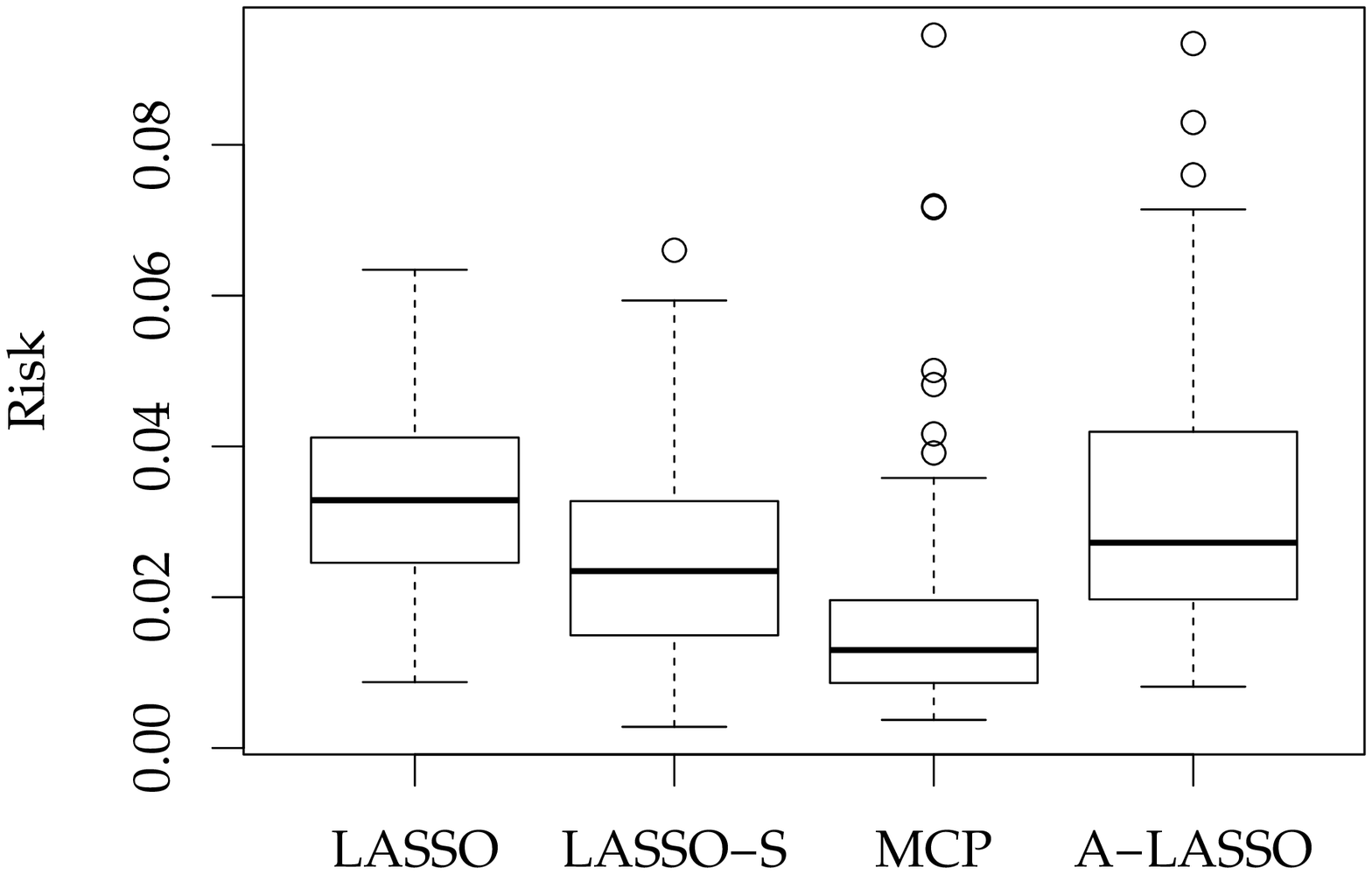}

     (b) $n=400$
    \end{center}
   \end{minipage}

\caption{Risk of selected model ($\tau=0.1$).}
\label{fig:papsim06risk}
 \end{center}
\end{figure}

\begin{figure}[p]
 \begin{center}
   \begin{minipage}[t]{80mm}
    \begin{center}
     \includegraphics[width=80mm]{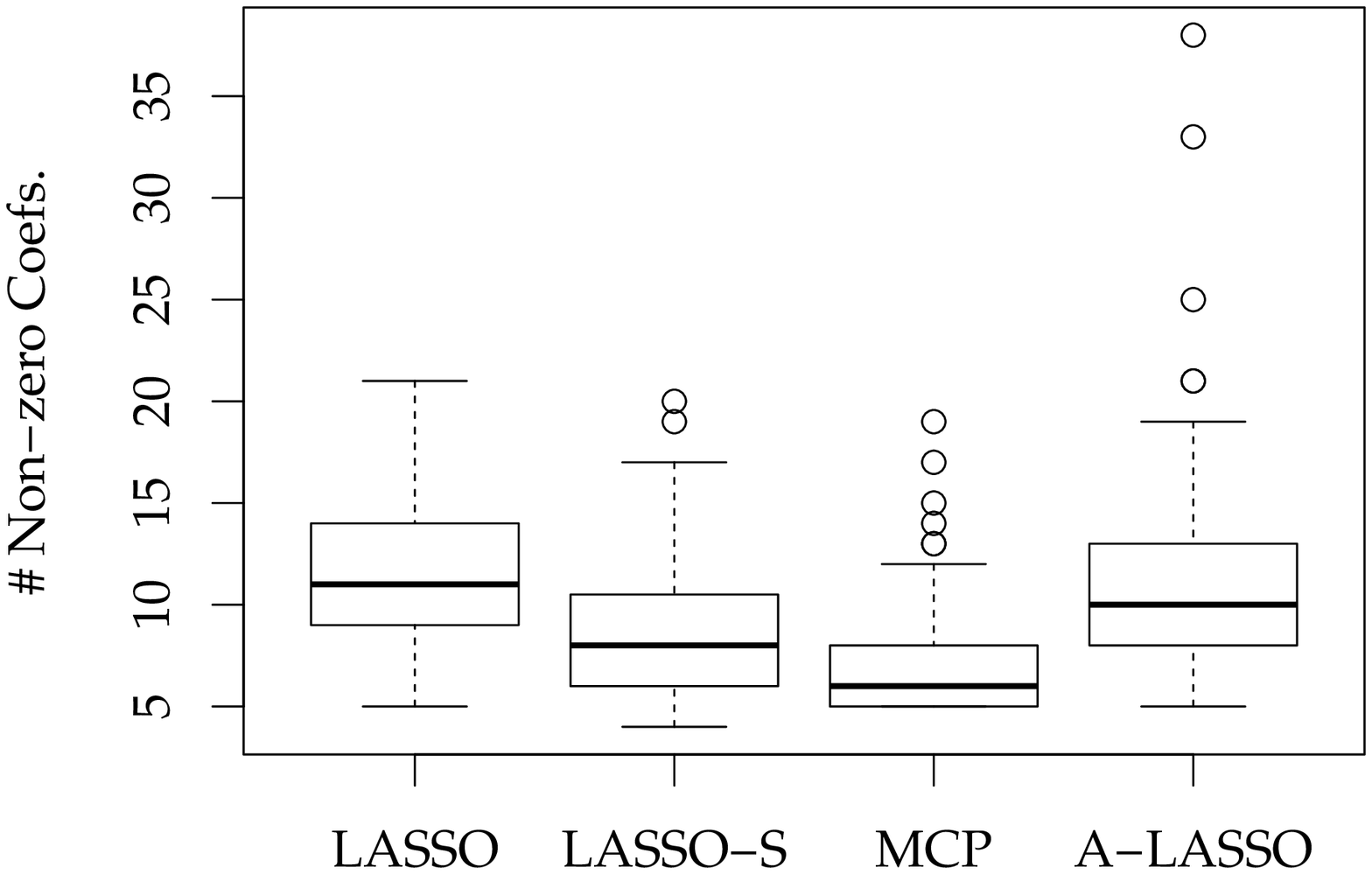}

     (a) $n=100$
    \end{center}
   \end{minipage}

   \begin{minipage}[t]{80mm}
    \begin{center}
     \includegraphics[width=80mm]{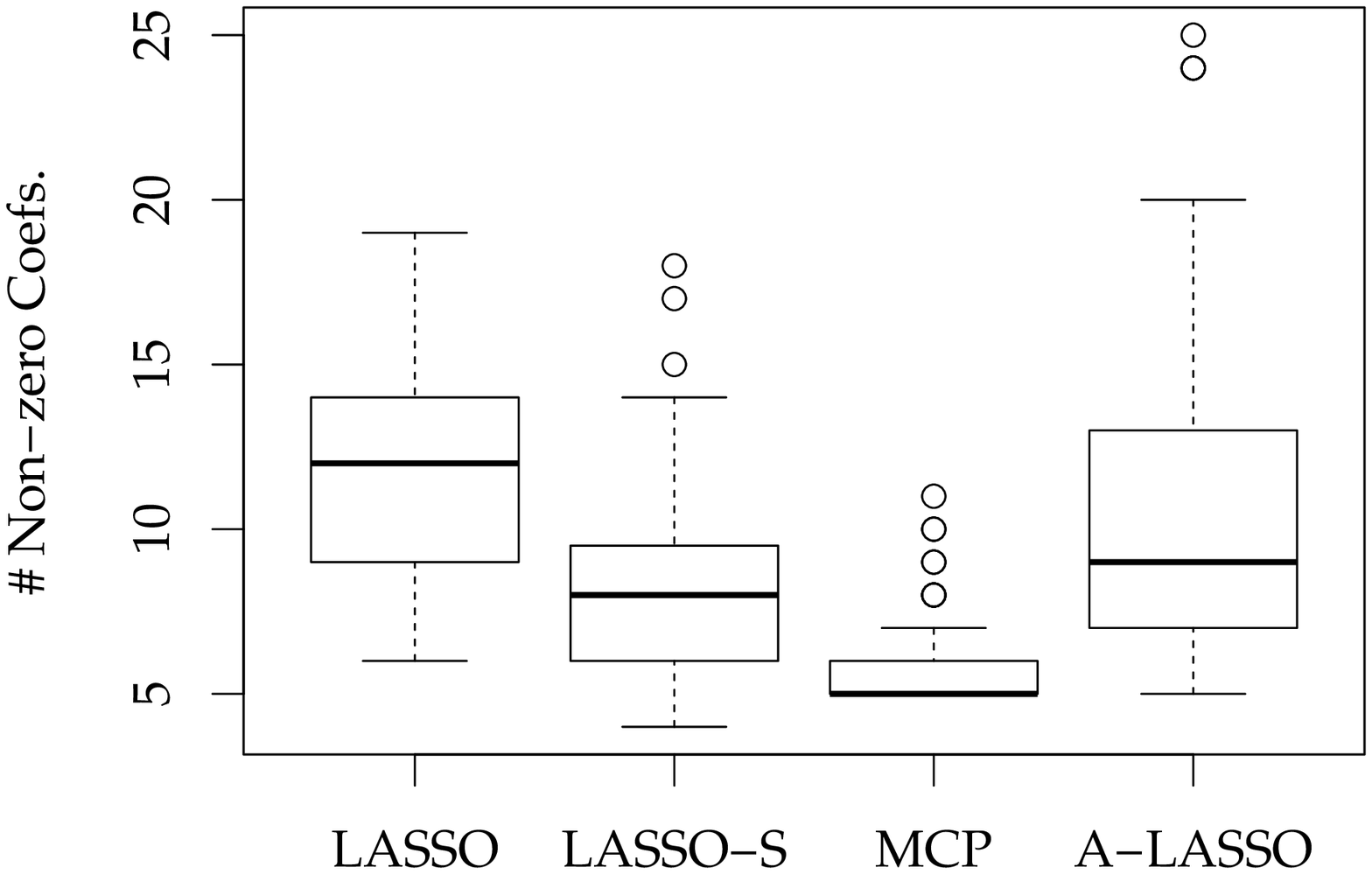}

     (b) $n=400$
    \end{center}
   \end{minipage}

\caption{The number of non-zero coefficients of selected model ($\tau=0.1$).}
\label{fig:papsim06df}
 \end{center}
\end{figure}

\begin{figure}[p]
 \begin{center}
   \begin{minipage}[t]{80mm}
    \begin{center}
     \includegraphics[width=80mm]{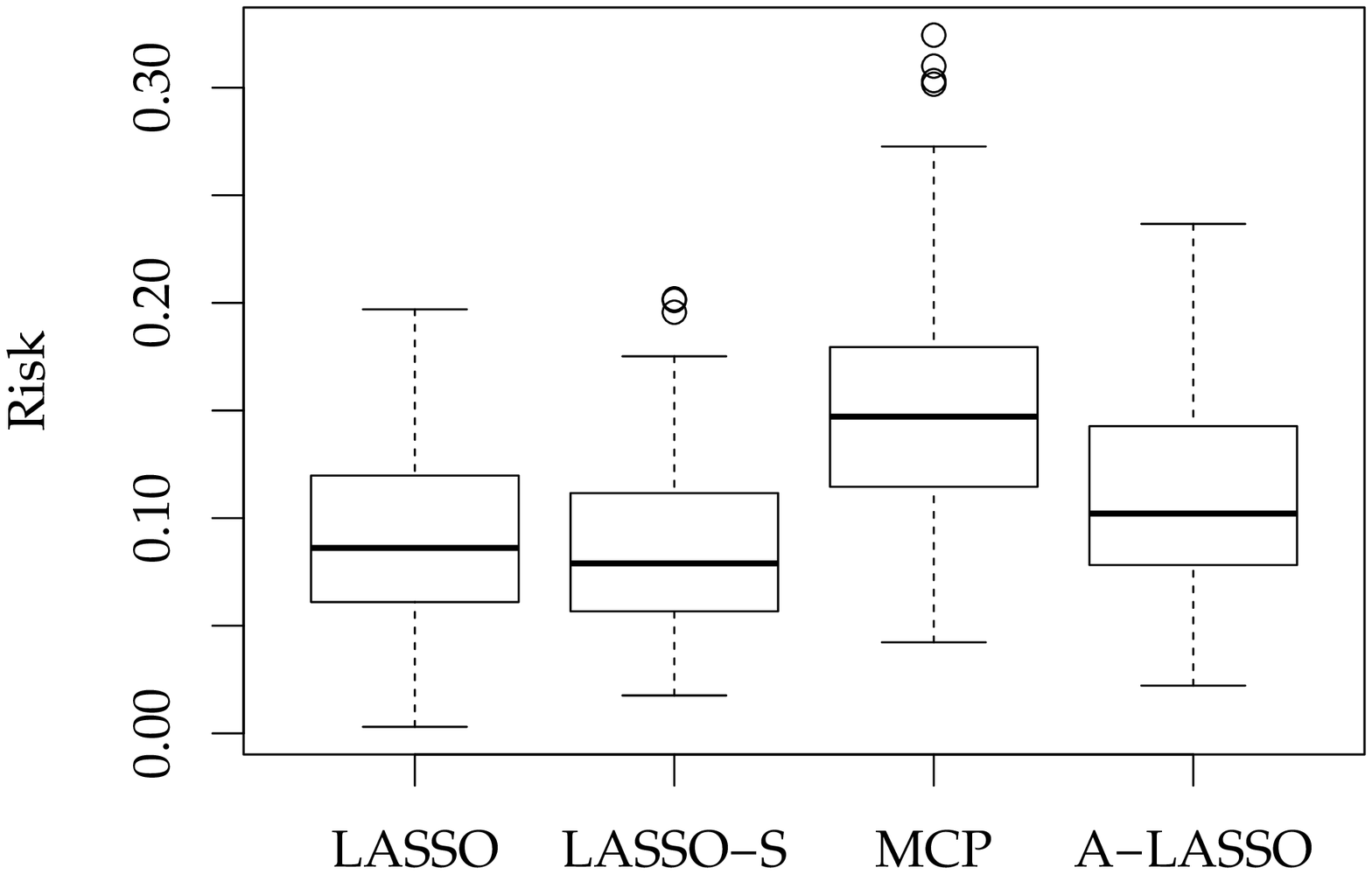}

     (a) $n=100$
    \end{center}
   \end{minipage}

   \begin{minipage}[t]{80mm}
    \begin{center}
     \includegraphics[width=80mm]{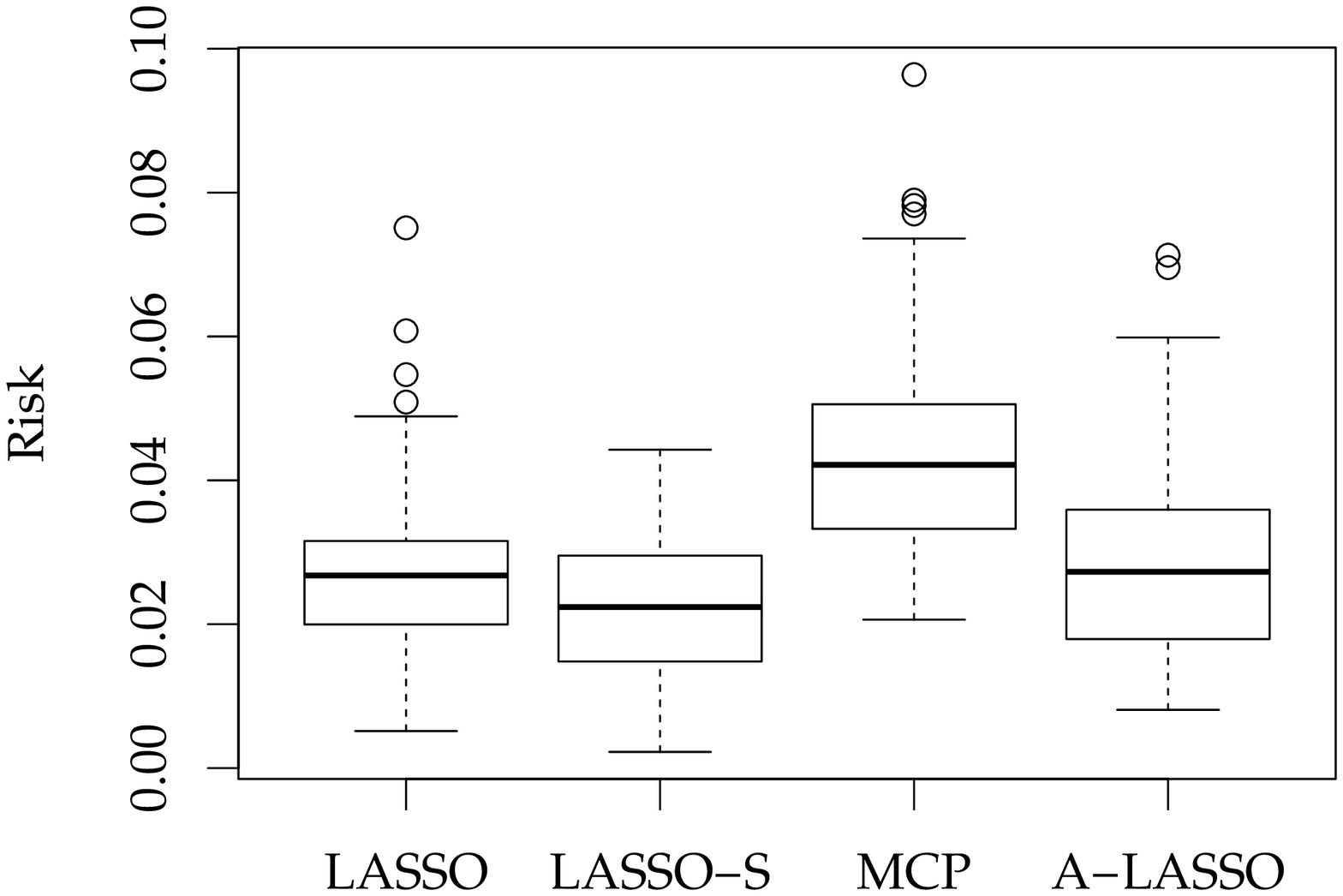}

     (b) $n=400$
    \end{center}
   \end{minipage}

\caption{Risk of selected model ($\tau=0.4$).}
\label{fig:papsim07risk}
 \end{center}
\end{figure}

\begin{figure}[p]
 \begin{center}
   \begin{minipage}[t]{80mm}
    \begin{center}
     \includegraphics[width=80mm]{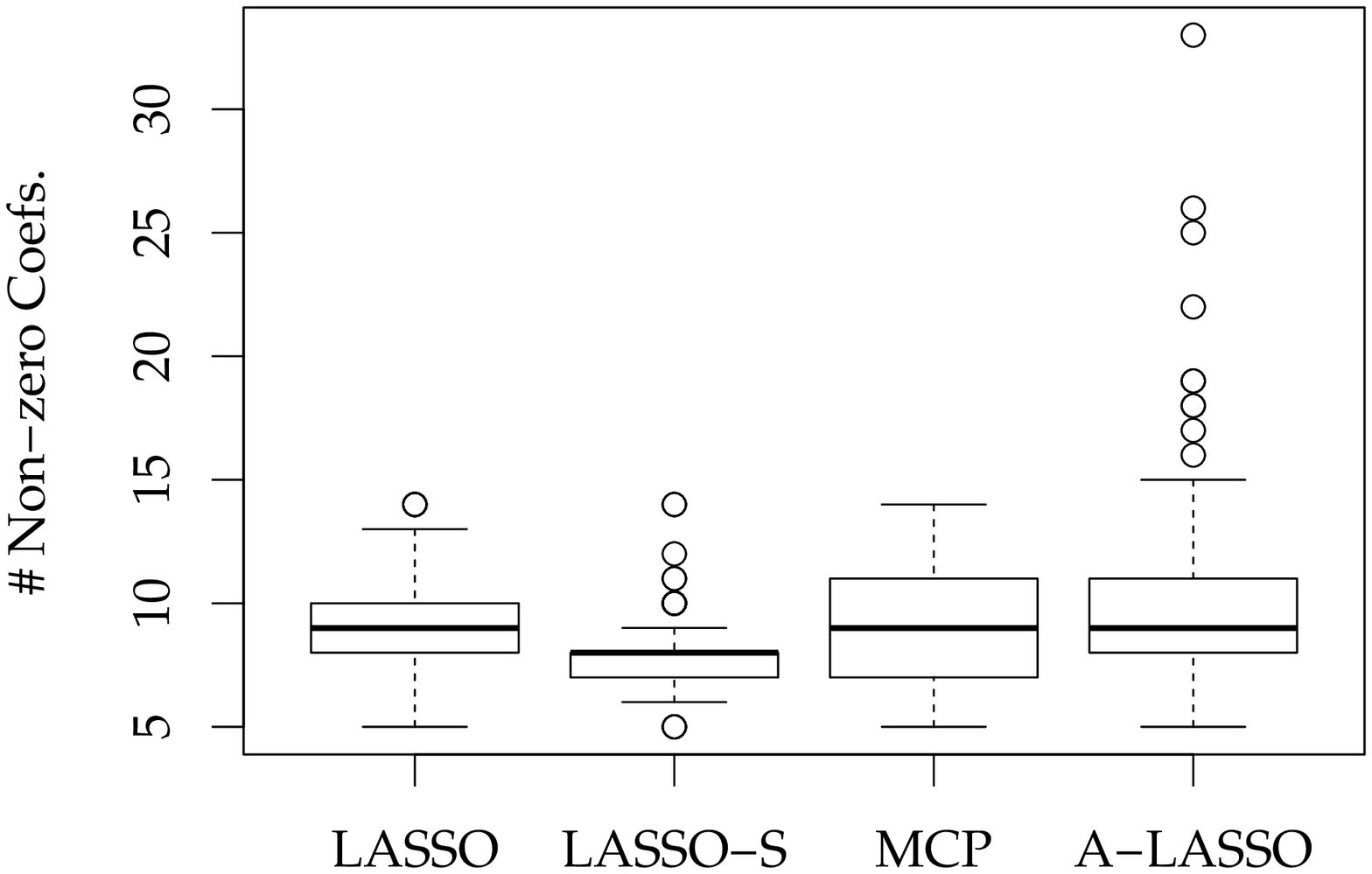}

     (a) $n=100$
    \end{center}
   \end{minipage}

   \begin{minipage}[t]{80mm}
    \begin{center}
     \includegraphics[width=80mm]{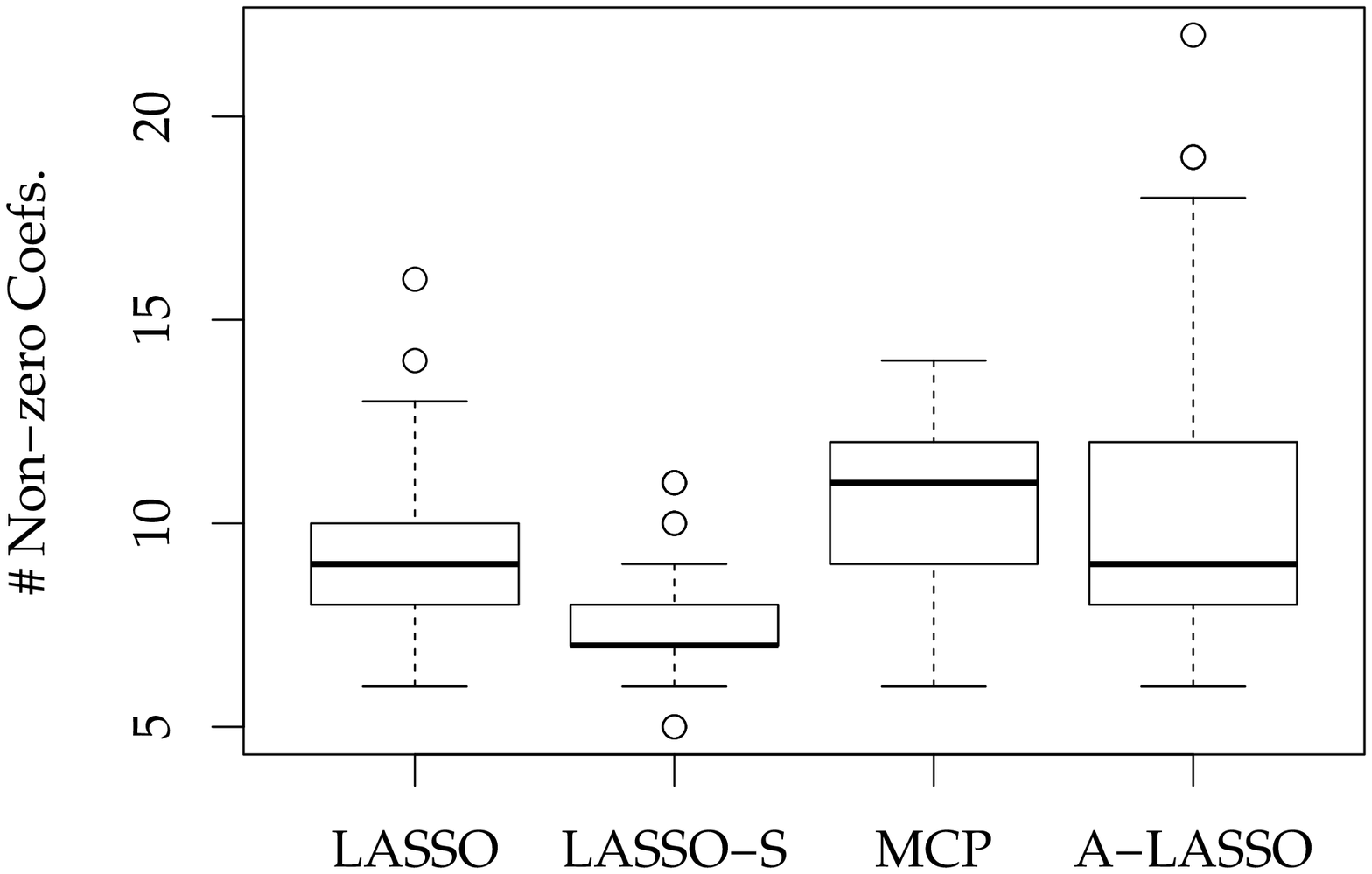}

     (b) $n=400$
    \end{center}
   \end{minipage}

\caption{The number of non-zero coefficients of selected model ($\tau=0.4$).}
\label{fig:papsim07df}
 \end{center}
\end{figure}

\subsection{Comparison to the other methods}

We here compare LASSO, LASSO-S, MCP and A-LASSO in the previous setting
of experiment although we test the cases of $n=100$ and $n=400$.
We use ``glmnet'' package\cite{glmnet} for LASSO, LASSO-S, A-LASSO and
``ncvreg'' package\cite{ncvreg} for MCP in R. We conduct simulations of model
selection in which the number of simulations is $S=100$. Basically, in
all methods, the candidate values of the regularization parameter is
$20$ points in $[0.01,0.5]$ with log-scale. In LASSO and LASSO-S, we
employ SUREs with $\esigma^2_{\rm CE}$ for model selection. In A-LASSO,
the weight for the penalty term is set to the reciprocal of the absolute
value of the ridge estimator. This is a substitute of the least squares
estimator to avoid a collinearity problem. The ridge parameter in doing
this is selected by $10$-fold cross validation in $10$ points in
$[0.01,10]$ with log-scale. By using this initial estimator, the
regularization parameter of A-LASSO and $\gamma$-parameter (exponent of
weights) are selected by a grid search of $10$-fold cross validation in
which the candidate values for $\gamma$-parameter are
$\{0.5,1.0,2.0\}$. For MCP, the regularization parameter and
$\gamma$-parameter are selected by a grid search of $10$-fold cross
validation, in which the candidate values of $\gamma$-parameter are
$\{2.5,3.0,3.5,4.0\}$. In MCP, the choice of $\gamma$-parameter seems to
largely affect the generalization performance. At each simulation, we
calculate the number of non-zero coefficients and actual risk of a
selected model. The boxplots of risk and the number of non-zero
coefficients of a selected model is depicted in
Fig.\ref{fig:papsim06risk} and Fig.\ref{fig:papsim06df} for $\tau=0.1$
and Fig.\ref{fig:papsim07risk} and Fig.\ref{fig:papsim07df} for
$\tau=0.4$.

In Fig.\ref{fig:papsim06risk} and Fig.\ref{fig:papsim06df}, we can see
that LASSO-S tends to select a sparse model with lower risk in comparing
with LASSO. Especially, selection of a sparse representation of LASSO-S
is notable. This shows that our scaling method surely contributes to
improve model selection property even though it is a simple modification
of LASSO. Therefore, the introduction of scaling really
solves the bias problem of LASSO. LASSO-S is also comparable or superior
to A-LASSO in terms of both of sparseness and risk even though we choose
the hyper-parameters in A-LASSO by cross validation. MCP shows the best
performance in sparseness and risk. This is notable when $n=400$,
relatively large sample case. However, when $n=100$, risk of MCP tends
to be larger than the other methods in some data.

On the other hand, as mentioned above, Fig.\ref{fig:papsim07risk} and
Fig.\ref{fig:papsim07df} show results when components in a target
function are relatively correlated. In this case, we can see that MCP
shows a worse total performance compared to the other methods even when
$n=400$. Contrastly, LASSO-S shows the best performance while LASSO also
shows a good performance. These results tell us that LASSO-S bring us a
stable improvement of LASSO regardless the number of
samples and condition on a target function. Additionally, both of
optimization and model choice of LASSO-S is very simple and fast.

\section{Conclusions and future works}

LASSO is known to be suffered from a bias problem that is caused by
excessive shrinkage. In this article, we considered to improve it by a
simple scaling method. We gave an appropriate empirical scaling value
that expands LASSO estimator and actually moves LASSO estimator close to
the least squares estimator of the post estimation. This is shown to be
especially effective when the regularization parameter is large; i.e. a
sparse representation. Since it can be calculated based of LASSO
estimator, we just run a fast and stable LASSO optimization procedure
such as LARS-LASSO or coordinate descent. We also derived SURE under the
modified LASSO with scaling. This analytic solution for model selection
is also a benefit of the proposed scaling method. As a result, we gave a
fully empirical sparse modeling procedure by a scaling method. In a
simple numerical example, we verified that the proposed scaling method
actually fixes the problem in LASSO and has a stability of model
selection compared to MCP and adaptive LASSO. As a future works, we need
more application results of our scaling method. Although we considered
to assign a single scaling value for all coefficients in this article,
the assignment of coefficient-wise scaling values is expected to improve
a prediction performance. This extension of our scaling method is also a
part of future works.

\section*{Acknowledgements}

This work was supported in part by Japan Society for the Promotion
of Science (JSPS) KAKENHI Grant Number 18K11433.



\end{document}